\documentclass{llncs}
\usepackage{llncsdoc}
\usepackage{amssymb}
\usepackage{amsfonts}
\usepackage{amsmath}
\usepackage{latexsym}
\usepackage[center]{subfigure}
\usepackage{epsfig}
\usepackage{hyperref}
\usepackage[table]{xcolor}
\definecolor{lightgray}{gray}{0.95}

\newtheorem{fact}[theorem]{Fact}

\title{From Exact Learning to Computing Boolean Functions and Back Again}
\author{Sergiu Goschin}
\institute{Department of Computer Science\\ Rutgers University, USA\\
\email{sgoschin@cs.rutgers.edu}
}
\begin{document}

\maketitle

\begin{abstract}
The goal of the paper is to relate complexity measures associated with the evaluation of Boolean functions (certificate complexity, decision tree complexity) and learning dimensions used to characterize exact learning (teaching dimension, extended teaching dimension). The high level motivation is to discover non-trivial relations between exact learning of an unknown concept and testing whether an unknown concept is part of a concept class or not. Concretely, the goal is to provide lower and upper bounds of complexity measures for one problem type in terms of the other.
\end{abstract}

\section{Introduction}
The problem of learning a function from a concept class should be connected or easy to relate to the problem of deciding whether the function is in the function class or not.  Imagine that one searches an element in an ordered set: the time it takes to find the element in the worst case scales with the logarithm of the size of the set. But the same fact is true (in the worst case) for testing whether the element is in the set or not. Thus, in exact learning, a hypothesis space could be viewed as playing the role of a (partially) ordered set and a target function as having the role of the element that is searched or tested for membership.

Our main goal is to discuss relations between learning and computing Boolean functions in a setting where a friendly 'teacher' provides the shortest proofs to exactly identify a function in a class or to evaluate\footnote{We will use the terms 'evaluate' and 'compute' interchangeably in the paper.} it. On the learning side this protocol is known as exact learning with a teacher \cite{gol92} while on the computational side it is known as the non-deterministic decision tree model \cite{buh00}. We will focus both on the worst case versions of the complexity measures for exact learning and computation of Boolean functions and their average case counterparts (\cite{kus96}, \cite{lee07}) as it has been observed that the worst case complexity measures are sometimes unreasonably large even for simple concept classes. 

A natural way to interpret the non-deterministic decision tree model and the protocol for learning with a teacher is as best case (but non-trivial) scenarios for evaluation and learning. We are interested in these protocols as any hardness results in such settings establish a natural limit for any other evaluation or exact learning protocol. That being said, we will also investigate the aforementioned relations in a setting where the agent has the power to do queries. In this context, we will briefly study the relations between the decision tree complexity of Boolean functions (on the evaluation side) and the query complexity (usually measured using a combinatorial measure called extended teaching dimension \cite{heg95}) of exact learning with membership queries \cite{ang88}. 

{\bf Motivation} Our motivation is two-fold. From a purely theoretical perspective we think it is interesting to formally study relations between combinatorial measures like teaching dimension and certificate / decision tree complexity as such relations could provide useful tools for proving lower and upper bounds in learning theory.

From a more applied perspective, the motivation is very similar to the one connecting learning and property testing (viewed as a relaxation of the learning problem \cite{ron08}). The intuition is that evaluating whether a particular concept is part of a concept class or is 'far' from being in the concept class should be an easier problem than learning the concept accurately. Thus, if multiple hypothesis classes are candidates for being parent classes for the target concept (like in agnostic learning), it might be worth running a testing algorithm before actually learning the concept to determine which function class to use as a hypothesis space or, alternatively, which function classes to eliminate from consideration.  

{\bf Related Work.} On the learning side, since the introduction of the teaching protocol and its associated notion of teaching dimension \cite{gol92}, there were several papers that described bounds for this complexity measure for various concept classes: (monotone) monomials, (monotone) DNFs, geometrical concepts, juntas, linear threshold functions (\cite{ant95}, \cite{lee07}). One of the early observations was that sometimes, even for simple concept classes, the worst case teaching dimension was trivially large, which contradicts the intuition that 'teaching' should be relatively 'easy' for naturally occurring concept classes. There are several relatively recent attempts (\cite{bal08}, \cite{zil11}) to change the model so as to better capture this intuition by allowing the learner or the teacher to assume more about each other.

Another perspective on better capturing the overall difficulty of learning a concept class in the teaching model is to consider the average case version of the teaching dimension. The general case for a function class of size $m$ was solved by \cite{kus96} who proved that $O(\sqrt{m})$ samples are enough to learn any function class, while there exists a function class for which $\Omega(\sqrt{m})$ sample are necessary. For particular concept classes, somewhat surprisingly, the average case bounds are actually much smaller: some hypothesis classes (DNFs \cite{lee07}, LTFs \cite{ant95}) have bounds on the average teaching dimension that scale with $O(\log(m))$ while others (juntas \cite{lee07}) are even independent of $m$. 

One intuitive reason for these gaps is that the general case upper bound is actually uninformative for large concept classes (when $m$ is large---$m > |X|^2$---, a better upper bound is the trivial $|X|$ that shows the learner all instances), whereas the proofs for particular concept classes actually take advantage of the specific structure of a class to derive meaningful upper bounds.

On the computational side, the (non-)deterministic decision tree model is relatively well understood (see \cite{buh00} for an excellent survey on the topic). Complexity measures like certificate complexity, sensitivity, block  sensitivity, decision tree complexity are used to quantify the difficulty of evaluating a Boolean function when access to the inputs is provided either by an 'all-knowing teacher' (the non-deterministic decision tree model) or via a query oracle (the deterministic decision tree model). While most of the results deal with the worst case versions of the aforementioned complexity measures, bounds for some of their average case versions appeared in the literature. Among them, we mention the results from \cite{ber96} which addresses the problem of the gap between average block sensitivity and average sensitivity of a Boolean function---a well known open problem for the worst case versions of the complexity measures.

{\bf Contributions.} The first result (Section \ref{S_Positive}) is that the teaching dimension and the certificate that a function is part of a hypothesis class (i.e. $1$-certificate complexity) play a dual role: when a class is 'easy' to teach it is 'hard' to certify its membership and vice-versa. The second contribution (Section \ref{S_Negative}) is to give lower bounds for the general case of the average non-membership certificate size. The results have several applications to learning and computing Boolean functions. Finally, we will describe structural properties of Boolean functions that point to connections between learning and computation in a setting that relates the (easier) teaching model with the (harder) query model (Section \ref{S_MQ}).

\section{Setting and Notation}

Let $\mathcal{F} = \{f_i\}_{i \in [m]}$ with $f_i : X \rightarrow \{0,1\}$ a class of $m$ Boolean functions and let $\mathcal{CF} = \{f:X\rightarrow \{0,1\} | f \notin \mathcal{F}\}$ be its complement  (we denote $[m]=\{1..m\}$). $\mathcal{F}$ itself can be seen as a Boolean function, $\mathcal{F} : 2^{X} \rightarrow \{0,1\}$, $F(f)=1$ iff $f \in \mathcal{F}$. We will usually consider $X$ to be $\{0,1\}^n$ and we will label elements $x \in X$ as instances or examples and $x^{(i)}, i \in [n]$ as the $n$ Boolean variables that describe $x$. In what follows, it is assumed that both the nature and the agent know $\mathcal{F}$, with nature choosing $f_t \in \{0,1\}^{2^n}$ in an adversarial manner while the agent is not aware of the identity of $f_t$. We will first describe the learning problem, then the computation problem and then discuss how they are related.

{\bf Learning with a Teacher.} On the learning side, we will focus on exact learning with a teacher in the loop. In this protocol, nature chooses $f_t \in \mathcal{F}$ and the learner knows $f_t$ is in $\mathcal{F}$ but is not aware of its identity. The learner receives samples (pairs of $(x,label(x))$ with $x \in X$) from a 'teacher', without knowing whether the teacher is well-intentioned or not. The goal of the learner is to uniquely identify the hidden function $f_t$ using as few samples as possible. The teacher is an optimal algorithm, aware of the identity of $f_t$, that gives the learner that most informative set of instances so that the learner uniquely identifies the target concept as fast as possible. The teacher is not allowed to make any assumptions about the learning algorithm, other than assuming it is consistent (i.e. that it maintains a hypothesis space consistent with the set of revealed samples). 

In this protocol, learning stops when the consistent hypothesis space of the learner has size $1$ and thus only contains the target hypothesis. For the purpose of this paper, the learner and the teacher are assumed to have unbounded computational power to compute updates to the hypothesis space and optimal sample sets (computational issues are treated in \cite{ser01} and \cite{gol92}). One intuitive perspective for this learning protocol is that it is the best case scenario of exact learning with membership queries \cite{ang88}, where the learner always guesses the best possible queries to find the target hypothesis. In the model of {\bf exact learning with membership queries}, the teacher is removed from the protocol, and the learner is responsible for deciding which inputs to query for labels with the same goal of minimizing the number of samples until the target concept is discovered.

We will now define a complexity measure (the teaching dimension) for learning a fixed function $f \in \mathcal{F}$ in the protocol of learning with a teacher.

\begin{definition}
For a fixed $f \in \mathcal{F}$, a minimum size teaching set $TS(f)$ is a set of samples that uniquely identifies $f$ among all functions in $\mathcal{F}$ with a size that is minimal among all possible teaching sets for $f$. The teaching dimension of $f$ (with respect to $\mathcal{F}$) is $TD_\mathcal{F}(f) = |TS(f)|$. The teaching dimension of $\mathcal{F}$ is $TD(\mathcal{F}) = max_{f \in \mathcal{F}} TD(f)$.
\end{definition}

Intuitively, a teaching set is a shortest 'proof' that certifies the identity of the initially hidden target concept. The teaching dimension is simply the maximum size of such a 'proof' over the entire hypothesis space. 

To capture the difficulty of learning a hypothesis class as a whole we will define average $TD(\mathcal{F})$, which has some interesting combinatorial properties (\cite{kus96}).

\begin{definition}
The average teaching dimension of $\mathcal{F}$ is $aTD(\mathcal{F}) = \frac{\sum_{f\in \mathcal{F}} TD_\mathcal{F}(f)}{|\mathcal{F}|}$.
\end{definition}

{\bf Computation in the Decision Tree Model.}\label{S_Evaluation}
On the evaluation side, we will focus on 'proofs' that certify what is the value of a Boolean function $f:X \rightarrow \{0,1\}$ on an unknown input $x \in X$ (with $X$ usually $\{0,1\}^n$). 
We will thus focus on {\bf certificate complexity}, which quantifies the difficulty of computing a Boolean function in the non-deterministic decision tree computation model. 

Let's assume $f$ is fixed and known to both the nature and the agent. The protocol of interaction is as follows: nature chooses an input $x \in X$ ($f(x) = b$ with $b=0$ or $b=1$) without revealing it to the agent, and offers query access to the bits that define $x$. For any query $i$, it reveals the correct bit value $x_i$ of the previously unknown bit $i$ in $x$. Now we can define certificate complexity for a fixed input and for a function:

\begin{definition}
For a fixed function $f$ and a fixed and unknown input $x$ with $f(x)=b$, a minimal $b$-certificate of $f$ on $x$ is a minimal size query set that fixes the value of $f$ on $x$ to $b$. The {\bf $b$-certificate complexity} $C^b_f(x)$ is the size of such a minimal query set.
\end{definition}

\begin{definition}
The $1$-certificate complexity of a Boolean function $f$ is \\ $C^1(f) =\max_{x \in X^1} C^1_f(x)$,  where $X^1 = \{x \in X | f(x)=1\}$. Symmetrically, $C^0(f) = \max_{x \in X^0} C^0_f(x)$. And then $C(f) = \max (C^1(f),C^0(f))$.
\end{definition}

An intuitive way to interpret certificate complexity is that it quantifies what is the minimal number of examples (pairs $(x_i,$value of $x$ on $x_i))$ a friendly 'teacher' (that knows $x$) must reveal to certify to an agent what is the value of $f$ on $x$.

While there is no previous definition for the notion of average certificate complexity in the literature, it is natural to define it in a similar manner (and for similar reasons) as for the average teaching dimension:

\begin{definition}
The average $1$-certificate complexity of a Boolean function $f$ is $aC^1(f) = \frac{\sum_{x \in X^1} C^1_f(x)}{|X^1|}$. We can symmetrically define $aC^0(f)$ and $aC(f)$.
\end {definition}

We will now define block sensitivity, another well studied complexity measure for computing Boolean functions, as we will need it later in the paper. 

\begin{definition}
A Boolean function $f$ is sensitive to a set $S \subseteq [n]$ on $x$ if $f(x) \not= f(x^{\{S\}})$, where $x^{\{S\}}$ is the input $x$ with bits in $S$ flipped to the opposite values. Then the block sensitivity of $f$ on $x$, $BS_f(x)$ is the size of the largest set of disjoint sets $S_1, S_2, ... ,S_k$ with the property that $f$ is sensitive to each set $S_i, i \in [k]$ on $x$. Also, the block sensitivity of $f$ is the maximum block sensitivity over all inputs $x$: $BS(f) = max_{x \in X} BS_{f}(x)$.
\end{definition}

The definition of average block sensitivity is natural and follows similarly to the definition of average teaching dimension and average certificate complexity. It is worth noting though that Definition \ref{D_aBS} is the same as that introduced in \cite{ber96} (as other notions of average block sensitivity have been studied).

\begin{definition}\label{D_aBS}
The average block sensitivity of a Boolean function $f$ is $aBS(f) = \frac{\sum_{x\in X}BS_f(x)}{|X|}$.
\end{definition}

\subsection{Connecting Learning and Computation}\label{S_Connection}
If in section \ref{S_Evaluation} we set $X = \{0,1\}^{2^n}$ and we re-label $f$ as $\mathcal{F}$, we can interpret $x \in X$ as Boolean functions $f_i : \{0,1\}^n \rightarrow \{0,1\}$ (with each $x$ being a truth table and thus a complete description of $f_i$). Thus $\mathcal{F}$ is a complete description of a hypothesis class with $\mathcal{F}(f_i) = 1$ iff $f_i \in \mathcal{F}$. The interaction protocol for both learning and evaluation proceeds in the same manner: at each step, a teacher reveals the value of an unknown function $f$ on an input instance from $\{0,1\}^n$. This is the sense in which we connect exact learning with a teacher and evaluation of Boolean functions in the non-deterministic decision tree computation model.

To gain more intuition, if one imagines the function class as a matrix with the rows being all elements in $\{0,1\}^{2^n}$ and the columns being the inputs in $\{0,1\}^n$, then, fixing the interaction protocol to contain an optimal teacher aware of the identity of a hidden row, learning is about identifying the hidden row among the subset of rows that determine $\mathcal{F}$ to be $1$, while evaluation is about determining whether a hidden row is part of a chosen subset of rows (that define $\mathcal{F}$) or not.

Another intuitive perspective is through the lens of hypergraphs: fixing a vertex set, learning is about identifying a hidden edge from a set of edges that form a hypergraph (the function class $\mathcal{F}$) while evaluation is about determining whether a given subset of vertices is an edge in the hypergraph or not.

\subsection{Simple Examples}\label{S_SimpleExamples}
In this section we will describe bounds for several simple concept classes with the goal of building intuition and exhibiting extreme values for $C^0$, $C^1$ and $TD$.

{\bf Powerset}. This class is a trivial example for which $\mathcal{F_P} = 1$ (all functions defined on $\{0,1\}^n$ are part of the concept class). It is easy to see that $TD(F_P) = aTD(F_P) = 2^n$ since to locate a particular concept one needs to query all examples (otherwise there will be at least two concepts that are identical on all previous instances). 
And $C^0(f) = C^1(f) = 0, \forall f$ since $\mathcal{F_P}$ is constant.

{\bf Singletons}. $\mathcal{F_S} = \{f_i\}_{i \in [2^n]}$, with $f_i(x) = 1$ iff $x=\mathbf{0}^{i \rightarrow 1}$  (i.e. the all-0 vector with the $i$-th coordinate flipped to $1$). Then $TD(\mathcal{F_S}) = aTD(\mathcal{F_S}) = 1$ because it is enough to show the $1$-bit of the target function to uniquely identify it among all $f_i$. Nevertheless certifying that an $f \in \{0,1\}^{2^n}$ is (or isn't) in $\mathcal{F_S}$ is hard in the worst case. If nature chooses the all-$0$ function $f_0$ as a target, certifying that $f_0$ is not part of $\mathcal{F_S}$ will require seeing all $2^n$ inputs as, at any intermediate time, there will be at least a function in $\mathcal{F_S}$ consistent with $f_0$. So $C^0(\mathcal{F_S}) = 2^n$. Similarly $C^1(\mathcal{F_S}) = 2^n$.

{\bf Singletons with empty set}. $\mathcal{F_{SE}} = \mathcal{F_S} \cup f_0$. For this function class, teaching becomes hard, as teaching $f_0$ requires seeing all examples to differentiate it from the other functions. So $TD(\mathcal{F}) = 2^n$. Teaching the other functions is easy though, as showing the $1$-bit is enough to certify what the function is. So $aTD(\mathcal{F}) = 2$. The $0$-certificate is small as any function not in $\mathcal{F_{SE}}$ is evaluated to $1$ for at least two examples, so showing these two examples is enough to certify that the function is not in $\mathcal{F_{SE}}$ and thus $C^0(\mathcal{F_{SE}}) = aC^0(\mathcal{F_{SE}}) = 2$. The $1$-certificate is large though since for any $f \in \mathcal{F_{SE}}$, at least $2^n - 1$ examples must be shown, so $C^1(\mathcal{F_{SE}}) = aC^1(\mathcal{F_{SE}}) = 2^n-1$.

{\bf The dictator function}. $\mathcal{F_D} =  \{f_i\}_{i \in [2^{2^n - 1}]}$, with $f_i \in \mathcal{F_D}$ iff $f_i(x_j) = 1$ for some fixed $x_j \in \{0,1\}^n$ (half of the Boolean functions defined on $\{0,1\}^n$ are in $\mathcal{F_D}$). The learning problem is hard ($TD(\mathcal{F_D}) = aTD(\mathcal{F_D}) = 2^n-1$) as it reduces to learning a {\bf Powerset} class on $2^n - 1$ bits. However, $C^0(f) = C^1(f) = 1, \forall f$, since membership to $\mathcal{F_D}$ can be decided if the value of the function on $x_j$ is revealed.

\section{Teaching and Certifying Membership}\label{S_Positive}

In this section we will study connections between teaching a function in a hypothesis class (conditioned on knowing that the function is indeed in the class) and proving that the function is part of the class (with no prior knowledge other than the knowledge of $\mathcal{F}$). We will begin with a simple fact that is meant to illustrate the improvement in the next subsection:

\begin{fact}\label{F_UBLB}
For any fixed instance space $X$, any function class $\mathcal{F}$ and any $f \in \mathcal{F}, f:X \rightarrow \{0,1\}$, $0 \le TD_{\mathcal{F}}(f) + C^1_{\mathcal{F}}(f) \le 2|X|$.
\end{fact}

\subsection{A Lower Bound Technique}
The upper bound from Fact \ref{F_UBLB} is almost tight in the worst case---the example of {\bf Singletons with empty set} from section \ref{S_SimpleExamples} shows  that  the sum of the two quantities can be $2|X|-1$. The lower bound, on the other hand, is very loose as the next theorem shows:

\begin{theorem}\label{T_LB}
For any fixed instance space $X$, any function class $\mathcal{F}$ and any $f \in \mathcal{F}, f:X \rightarrow \{0,1\}$, $TD_{\mathcal{F}}(f) + C^1_{\mathcal{F}}(f) \ge |X|$.
\end{theorem}

\begin{proof}
Let's assume nature chooses $f_a \in \mathcal{F}$ as the target hypothesis but does not reveal to the agent whether $f_a \in \mathcal{F}$ or $f_a \notin \mathcal{F}$. Let $TS = TS_{\mathcal{F}}(f_a)$ be the minimal teaching set for $f_a$ and $CS = C^1_{\mathcal{F}}(f_a)$ be the smallest $1$-certificate for $f_a$.

Let's assume that the goal of the teacher is to reveal the identity of $f_a$ to the learner using samples $(x,f_a(x))$. But, since in this protocol the learner is not aware of whether $f_a$ is evaluated to $1$ or $0$ by $\mathcal{F}$, it has to be the case that to uniquely identify it, it must see the value of the function on all $x \in X$ (otherwise there will always be another function consistent with the examples seen so far that is evaluated to the opposite value by $\mathcal{F}$). This argument is equivalent to learning a function in the case of  the {\bf Powerset} function class from Section \ref{S_SimpleExamples}.

Now let's describe an alternative strategy for the teacher that has the same effect of uniquely identifying $f_a$. In the first epoch, the teacher reveals all the samples from $CS$. This will 'certify' to any consistent learner that $f_a \in \mathcal{F}$ (since no $f$ such that $\mathcal{F}(f)=0$ is consistent with $CS$). We are now in the standard exact learning setting where the agent knows $f_a \in \mathcal{F}$ but does not know its identity. In the second epoch, the teacher will reveal all the samples in $TS \setminus CS$ to the learner---since there is no point in presenting elements from their (possibly non-empty) intersection twice. This strategy will uniquely identify $f_a$ without any prior knowledge about its membership to $\mathcal{F}$.

But we know that to uniquely identify $f_a$ we need exactly $|X|$ samples. So it has to be the case that $|CS \cup (TS \setminus CS)| = |CS \cup TS| = |X| \le |CS| + |TS| = TD_{\mathcal{F}}(f_a) + C^1_{\mathcal{F}}(f_a)$.\qed
\end{proof}

Since the above relation holds for any function $f \in \mathcal{F}$ it must hold for the average and worst case values of the complexity measures:

\begin{corollary}\label{C_LB}
For any $X, \mathcal{F}, f \in \mathcal{F}$, $TD(\mathcal{F}) + C^1(\mathcal{F}) \ge |X|$ and $aTD(\mathcal{F}) + aC^1(\mathcal{F}) \ge |X|$.
\end{corollary}

\subsection{Certifying Membership is (Usually) Hard}
In this section we will present a first application of Theorem \ref{T_LB}. We will show that, for some of the standard concept classes in learning theory, certifying membership is hard, meaning all input variables need to be queried to determine whether an unknown function is part of the function class.

\begin{table*}[ht]\caption{Previous Results ($|\mathcal{\mathcal{F}}| = m$, $\mathcal{F}(f) = 1$, iff $f \in \mathcal{F}, f:\{0,1\}^n\rightarrow \{0,1\}$)}\label{T_PreviousResults}
\vspace{-5mm}
\begin{center}
%\rowcolors{1}{}{lightgray}
    \begin{tabular}{| >{\centering}m{2.7cm} | >{\centering}m{2.7cm} |>{\centering}m{2.7cm} |>{\centering}m{2.7cm} |}
    \hline
 	& TD($\mathcal{F}$) & aTD($\mathcal{F}$) & $C^1(\mathcal{F})$ \tabularnewline \hline
	Monotone Monomials &  \cellcolor{green!25} $\Theta(n)$  \cite{gol92} & $\Theta(n)$ \cite{lee07} & \cellcolor{red!25} $\Omega(2^n)$ \tabularnewline \hline
	Monomials & \cellcolor{green!25} $\Theta(n)$ \cite{gol92} & $\Theta(n)$ \cite{lee07} &  \cellcolor{red!25} $\Omega(2^n)$ \tabularnewline \hline
	Monotone k-term DNF & $n^k + k$ \cite{lee07} & \cellcolor{green!25} $O(kn)$ \cite{lee07} &  \cellcolor{red!25} $\Omega(2^n)$  \tabularnewline \hline
	k-term DNF &  & \cellcolor{green!25} $O(kn)$ \cite{lee07} &  \cellcolor{red!25} $\Omega(2^n)$  \tabularnewline \hline
	LTF & $\Theta(2^n)$  \cite{ant95} & \cellcolor{green!25} $[n+1, n^2]$  \cite{ant95} &  \cellcolor{red!25} $\Omega(2^n)$ \tabularnewline \hline
	k-Juntas & $O(k 2^k \log{n})$ $\Omega(2^k \log{n})$ \cite{lee07} & \cellcolor{green!25} $\Theta(2^k)$ \cite{lee07} &  \cellcolor{red!25} $\Omega(2^n)$  \tabularnewline \hline	
    \end{tabular}
\hfill{}
\end{center}
\vspace{-7mm}
\end{table*}

In table \ref{T_PreviousResults} we present known lower and upper bounds for $TD$ and $aTD$ for a few hypothesis classes encountered in learning theory. It is important to note that $aTD$ scales at most logarithmically with the size of the input space (constant for $k$-juntas and logarithmic for (monotone) conjunctions, (monotone) DNFs, LTFs). Thus, from Corollary \ref{C_LB}, for such concept classes $\mathcal{F}$ with $aTD(\mathcal{F}) = o(|X|)$, it follows that: $C^1(\mathcal{F}) \ge aC^1(\mathcal{F}) \ge |X| - aTD(\mathcal{F}) = \Omega(|X|)$.

\subsection{Sparse Boolean Functions are Hard to Compute}
In this section we will present another application of Theorem \ref{T_LB}. We will show that 'sparse' Boolean functions have large certificate complexity. By sparse we mean Boolean functions with a 'small' (roughly the size of $|X|$) Hamming weight (number of $1$'s) in the output truth table. The following theorem makes this precise:

\begin{theorem}
For any sets $X$ and $Y = 2^{X}$ and any Boolean function $\mathcal{F}$, $\mathcal{F} : Y \rightarrow \{0,1\}$, let $m = | \{f \in Y | \mathcal{F}(f)=1\} |$. If $m = o(|X|^2)$, then $C^1(\mathcal{F}) = \Omega(|X|)$.
\end{theorem}

\begin{proof}
We can interpret $\mathcal{F}$ as a function class in the same manner as described in section \ref{S_Connection}. Due to Theorem $1$ from \cite{kus96} we know there is a general upper bound on $aTD(\mathcal{F}) = O(\sqrt{m}) = o(|X|)$. Then $C^1(\mathcal{F}) \ge aC^1(\mathcal{F}) \ge |X| - aTD(\mathcal{F}) = |X| - o(|X|) = \Omega(|X|)$.
\qed
\end{proof}

It is worth mentioning that while the teaching dimension is used in the proof, the result is strictly about the hardness of computing $\mathcal{F}$.

\section{Teaching and Certifying Non-Membership}\label{S_Negative}
In this section we will study bounds and relations between learning a function in a class and proving that the function is not in the class. The high-level intuition of why the two problems are similar comes from the similarity with the problem of searching an element in an unordered / ordered set: in the worst case, searching an element that is part of the set is as hard as searching the element if it is not in the set. We will present a set of results that are a first indication that, at least in the average case, the learning problem and the non-membership decision problem are similar.

We will first deal with the worst case for $C^0$.
\begin{theorem}
For any instance space $X$, any function class $\mathcal{F}$, with $m = |\mathcal{F}| < |X|$, and any $f \not\in \mathcal{F}$, $C^0_{\mathcal{F}}(f) \le m$.
\end{theorem}

\begin{proof} From Bondy's theorem (\cite{juk01}, Theorem $12.1$), we know there exists a set of coordinates $TS(\mathcal{F})$ of size $m-1$ that, when revealed, uniquely identifies any target function $f \in \mathcal{F}$. Let's choose $f_a \in \mathcal{CF}$. 

Let's first assume $\exists f' \in \mathcal{F}$ that is consistent with $f_a$ on $TS$. If we reveal the labels for all the coordinates in $TS$, $f'$ will be the only function consistent with $f_a$. It must then be the case that revealing just another coordinate will lead to a certificate of size $m$ for $f_a$. If, on the other hand, for any $f \in \mathcal{F}$, $f$ is inconsistent with $f_a$ on $TS$, revealing the labels of all the coordinates in $TS$ will lead to an upper bound for the size of any certificate that $\mathcal{F}(f_a) =0$.\qed
\end{proof}

The following corollary follows immediately: 

\begin{corollary}\label{C_C0}
For any function class $\mathcal{F}$, with $m = |\mathcal{F}| < |X|$, $C^0(\mathcal{F}) \le m$.
\end{corollary}

The bound on $C^0(\mathcal{F})$ is tight as the example of {\bf Singletons} from \ref{S_SimpleExamples} demonstrates, which thus settles the worst case bounds for $C^0$.

\subsection{Bounds for $aC^0$}
We will begin by proving a lower bound for $aC$. The result is a simple application of a theorem from \cite{ber96} that puts a lower bound on the average block sensitivity of a Boolean function.

\begin{theorem}\label{T_aC}
For any set $X$, there exists a function $\mathcal{F}_a : 2^X \rightarrow \{0,1\}$ such that $aC(\mathcal{F}_a) \ge \sqrt{|X|}$.
\end{theorem}

\begin{proof}
We will choose $\mathcal{F}_a$ to be the Rubinstein function (more details in the proof for Theorem \ref{T_aC0}) and apply the result from Proposition $6$ in \cite{ber96} that $aBS(\mathcal{F}_a) \ge \sqrt{|X|}$. But, for any $x \in X$, $BS_{\mathcal{F}_a}(x) \le C_{\mathcal{F}_a}(x)$ since a certificate for an input must contain at least a bit from each sensitive block otherwise the value of the function can be flipped by an adversary (see \cite{buh00}, Proposition $1$). And so the relation must hold for the average case as well since $\sqrt{|X|} \le aBS(\mathcal{F}_a) = \frac{\sum_{x\in X} BS_{\mathcal{F}_a}(x)}{|X|} \le \frac{\sum_{\mathcal{F}_a}C_{\mathcal{F}_a}(x) }{|X|} = aC(\mathcal{F}_a)$.\qed
\end{proof}

Now we will prove that the same property holds even if we restrict our attention to $aC^0$ which requires more work.

\begin{theorem}\label{T_aC0}
For any set $X$, there exists a function $\mathcal{F}_a : 2^X \rightarrow \{0,1\}$ such that $aC^0(\mathcal{F}_a) = \Omega(\sqrt{|X|})$.
\end{theorem}

\begin{proof}
Let $|X| = 4k^2$. We will choose again $\mathcal{F}_a$ to be the Rubinstein function on $X$ and we define it as: the $4k^2$ variables are partitioned in $2k$ pieces of size $2k$ each, $\mathcal{F}_a$ is $1$ iff there exists at least one piece of the partition that has $2$ consecutive variables equal to $1$ and the rest $0$. We will count the number of inputs that are evaluated to $0$, i.e. $|X^0|$, with $X^0 = \{x \in X|\mathcal{F}_a(x) = 0\}$. 

If we fix a piece of the partition, there are $2k$ configurations of the input variables in that piece that lead to $\mathcal{F}_a$ being evaluated to $1$. And so $2^{2k}-2k$ configurations don't "contribute" into making $\mathcal{F}_a$ to be $1$. But, if in each piece there is a configuration that doesn't "contribute" to $\mathcal{F}_a = 1$, then $\mathcal{F}_a=0$. There are thus $(2^{2k}-2k)^{2k}$ such configurations.

We know from Theorem \ref{T_aC} that:

\begin{eqnarray*}
 2k = \sqrt{|X|} \le aC(\mathcal{F}_a) &=& \frac{\sum_{x \in X^1} C(x) + \sum_{x \in X^0} C(x)}{2^{4k^2}}\\
&\le& \frac{ (2^{4k^2} - (2^{2k} - 2k)^{2k})4k^2 + (2^{2k} - 2k)^{2k} aC^0(\mathcal{F}_a) }{2^{4k^2}}\\
&\le& (1 - (1 - \frac{2k}{2^{2k}})^{2k})4k^2 + (1-\frac{2k}{2^{2k}})^{2k}aC^0(\mathcal{F}_a)
\end{eqnarray*}
where the second inequality follows by upper bounding all $C(x), x \in X^1$ by the maximum possible certificate complexity, i.e. the size of $X$, $4k^2$. It can then be shown that:
%\begin{eqnarray*}
$\lim_{k \to \infty} (1 - (1 - \frac{2k}{2^{2k}})^{2k})4k^2 = 0$ and $\lim_{k \to \infty} (1-\frac{2k}{2^{2k}})^{2k} = 1$
%\end{eqnarray*}
and thus $aC^0(\mathcal{F}_a) = \Omega(2k)$.\qed
\end{proof}

Before we continue, it is interesting to remark the similarity (at a high level) of the result from Theorem \ref{T_aC0} with Theorem $1$ from \cite{kus96} that describes a lower bound of $\sqrt{|X|}$ on $aTD(\mathcal{F})$. The lower bound tools are very different (Rubinstein function for $aC^0$ and projective planes for $aTD$), but, since they lead to an identical lower bound for related complexity measures, it would be interesting to see if there are some deep connections between them.

We will now prove a (weak) lower bound tool for the relationship between $aTD$ and $aC^0$. 

\begin{theorem}
For any $\alpha < 2$, there exists a function class $\mathcal{F}_a$ such that $aTD(\mathcal{F}_a) + aC^0(\mathcal{F}_a) \ge \alpha |X|$.
\end{theorem}

The theorem states that $aTD$ and $aC^0$ can be simultaneously "large" ($\Theta(|X|)$), a statement that is not immediately obvious (at least given the simple concept classes considered in Section \ref{S_SimpleExamples}). 

\begin{proof}
Let's consider $\mathcal{F_C} = 1 - \mathcal{F_{SE}}$ (the complement of  the 'Singletons with empty set' concept class). Then $aTD(\mathcal{F_C}) = aTD(\mathcal{F_{SE}}) = |X| - 1$. Also $aC^0(\mathcal{F_C}) = aC^1(\mathcal{F_{SE}}) = |X|-1$. So $aTD(\mathcal{F_C}) + aC^0(\mathcal{F_C}) = 2|X|-2$ and thus there exists no $\alpha$ such that the sum is smaller than $\alpha |X|$. \qed
\end{proof}

\section{Connection with Membership Query Learning}\label{S_MQ}

For the purpose of this section we will focus on learning and computing with queries (the membership query learning and deterministic decision tree computation models) as this perspective will allow us to get more intuition about the structure of the function $\mathcal{F}$.

In a manner similar to the way we have defined teaching dimension for the protocol of exact learning with a teacher, we will define $MEMB(\mathcal{F})$ to be the (worst case) optimal learning bound for learning a function class $\mathcal{F}$ in the exact learning with membership queries protocol. Also, in a similar manner as for the certificate complexity definition in the non-deterministic decision tree model, we define $D(\mathcal{F})$ to be the (worst case) optimal complexity of computing a Boolean function $\mathcal{F}$ in the deterministic decision tree model (for more formal definitions see \cite{ang88} and \cite{buh00}).
 
In \cite{heg95} Heged\H{u}s introduced a complexity measure for bounding $MEMB(\mathcal{F})$ called the {\bf Extended Teaching Dimension}, which, as the name suggests, is inspired by the definition of the {\bf Teaching Dimension}. We will define this complexity measure and then describe a result that establishes a connection between $TD(\mathcal{F}),C^0(\mathcal{F})$ and $MEMB(\mathcal{F})$. 

\begin{definition}[\cite{heg95}]
A set $S \subseteq X$ is a {\bf specifying set} (SPS) for an arbitrary concept $f \in 2^X$ with respect to the hypothesis class $\mathcal{F}$ if there is at most one concept in $\mathcal{F}$ that is consistent with $f$ on $S$. Then the {\bf Extended Teaching Dimension} ($ETD$) of $\mathcal{F}$ is the minimal integer $k$ such that there exists a specifying set of size at most $k$ for any concept $f \in 2^X$. 
\end{definition}

\begin{theorem}\label{T_ETD}
For any function class $\mathcal{F}$, $\max\{TD(\mathcal{F}), C^0(\mathcal{F})\} \le ETD(\mathcal{F}) \le \max\{TD(\mathcal{F}), C^0(\mathcal{F}) \} + 1$.
\end{theorem}

\begin{proof}

A specifying set for any $f \in \mathcal{F}$ is also a teaching set for $f$, as it uniquely identifies the function among all other functions in $\mathcal{F}$. Also, a specifying set for any $f_c \in \mathcal{CF}$ is 'almost' a certificate that $f_c$ is not in $\mathcal{F}$ as it differentiates $f_c$ from all other functions in $\mathcal{F}$ with the exception of at most one function. Revealing an extra instance is thus sufficient to differentiate $f_c$ from all $f \in \mathcal{F}$ and thus obtaining a certificate for $f_c \in \mathcal{CF}$.

Let $ETD(\mathcal{F})=k$ for some fixed $k$. Then there must be at least a function $f_a \in 2^X$ that has a minimal specifying set of size exactly $k$. Let's assume, wlog, that such a function is unique. Let's first consider the case of $f_a \in \mathcal{F}$. Since $ETD(\mathcal{F})=k$, it means all $f \in \mathcal{F}$ have a teaching set of size $\le k$. $f_a$ can't have a teaching set with a size smaller than $k$ since such a teaching set would also be a specifying set of size $<k$ which is not possible given the assumption of uniqueness. And since $TD_F(f_a)=k$ it means $TD(\mathcal{F})=k$. Now let's pick an arbitrary $f_c \in \mathcal{CF}$. Since $f_a$ is the unique function with a specifying set of size $k$, it means that $|SPS_{\mathcal{F}}(f_c)| < k$ and thus $C^0(f_c) \le |SPS_{\mathcal{F}}(f_c)| + 1 \le k$ which thus proves the desired relation for this case.

The second case with $f_a \in \mathcal{CF}$ is treated similarly. \qed
\end{proof}

Theorem \ref{T_ETD} directly leads to the following corollary (since $ETD(\mathcal{F})$ is a lower bound for $MEMB(\mathcal{F}$)) stating that learning with membership queries is at least as hard as certifying non-membership:

\begin{corollary}
For any function class $\mathcal{F}$, $C^0(\mathcal{F}) \le MEMB(\mathcal{F})$.
\end{corollary}

\subsection{Is $\mathcal{F}$ weakly symmetric for natural learning problems?}
We will now give a result that follows from Theorem \ref{T_LB} and connects learning and computation in the query model:

\begin{corollary}\label{T_MQLB}
For any instance space $X$, any fixed function class $\mathcal{F}$ and any $f \in \mathcal{F}, f:X\rightarrow \{0,1\}$, $MEMB_{\mathcal{F}}(f) + D_{\mathcal{F}}(f) \ge |X|$.
\end{corollary}

The proof is immediate as the teaching dimension is a lower bound for the optimal membership query bound \cite{gol92} and the certificate complexity is a lower bound for the decision tree complexity \cite{buh00} and so we can just apply theorem \ref{T_LB} to get the desired relation.

A natural question is how useful is this bound for standard concept classes from learning theory. It is this question that we address in this subsection where we describe an interesting structural property of $\mathcal{F}$.

We will begin with a few (informal) definitions (see \cite{laz02} for a complete reference). In the deterministic decision tree computation model, a Boolean function is labeled {\bf evasive} if, in the worst case, all input variables need to be queried to determine the value of the function. Several results describe sufficient conditions for large classes of Boolean functions to be evasive. An interesting class of Boolean functions are {\bf graph properties}. A graph property is a class of graphs (on a fixed number of vertices) that remains unchanged for any permutation of the vertices (graph connectivity for example). The variables for a graph property are the possible edges of a graph.

By construction, a graph property can be encoded as a {\bf weakly symmetric} Boolean function on the edges. Weakly symmetric Boolean functions are a generalization of symmetric Boolean functions. A Boolean function is weakly symmetric if, for any pair of variables, there exists a permutation of all variables that permutes the variables in the pair, such that the function remains unchanged. 

Graph properties are weakly symmetric since all permutations on the vertex set induce a set of permutations of the edge set which leave the function unchanged. A general hardness result (the {\bf Rivest-Vuillemin theorem} \cite{riv76}) for computing weakly symmetric Boolean functions (and implicitly graph properties) states that any non-constant weakly symmetric Boolean function defined on a number of  variables that is the power of a prime number is evasive. 

This brings us to the point of connection with the Boolean function $\mathcal{F}$ as we've defined it in section \ref{S_Connection}. The intuition is that in the same way that permuting vertices doesn't change a graph property, the input variables for $f$'s don't change the Boolean function $\mathcal{F}$, or in other words the definition of a function class ($\mathcal{F}$) is invariant to permutations of $x_a^{(i)}$ (the bits of $x_a \in X$). Moreover, by construction, $\mathcal{F}$ has a number of inputs which is a power of a prime ($2^n$ - all the possible inputs that can be defined on the original $n$ input variables) and natural concept class are not trivial.

So, if and when the intuition that $\mathcal{F}$ is weakly symmetric is correct, we can actually apply the aforementioned result to show that $\mathcal{F}$ is evasive and in turn that $D(\mathcal{F}) = \Omega(|X|)$ (and implicitly that $C(\mathcal{F}) = \Omega(\sqrt{|X|})$). In such situations the bound from Theorem \ref{T_MQLB} is not very useful as it puts no constraints on the optimal membership query bound. 

Interestingly though, the above intuition is false in general. For example the following theorem shows a natural concept class that leads to a function $\mathcal{F}$ that is not weakly symmetric. 

\begin{theorem}
If $m\mathcal{F}_k$ is the class of monotone monomials of size exactly\footnote{A Boolean function is representable by a monomial of size exactly $k$ if it has a monomial representation of size $k$ and no monomial representation for any $k' < k$.} $k$, $m\mathcal{F}_k$ (viewed as a Boolean function with input 'bits' from $X$ and inputs from $2^X$) is not weakly symmetric.
\end{theorem}

\begin{proof}
Let $X = \{0,1\}^n$ and for any $x \in X$ let $|x| = |\{i|x^{(i)} =1\}|$ be the weight of $x$ (the number of bits in $x$ that are $1$). 

Let's consider $f_A \in m\mathcal{F}_k$ and $x_a \in X$ such that $f_A(x_a) = 1$. Then it must be the case that $|x_a| \ge k$ (there exist $k$ bits among the $n$ bits of $x_a$ that are $1$). Let's consider $x_b$ be an input in $X$ such that $|x_b| < k$. Then $f_A(x_b) = 0$ since $x_b$ can't encode a monotone monomial of size $k$. 

Now let's consider an arbitrary permutation $\pi$ that changes $x_a$ with $x_b$. This means that $\pi$ will induce a Boolean function $f^{\pi}$ that will be evaluated to $1$ for $x_b$. But such a function can't be a monotone monomial of size exactly $k$ since $|x_b| < k$ and can't be evaluated to $1$. This means that any permutation that changes $x_a$ and $x_b$ will change the function $m\mathcal{F}_k$. So we have found a pair of variables for which no permutation of the other variables (that permutes the two) leaves $m\mathcal{F}_k$ unchanged. Thus $m\mathcal{F}_k$ can't be weakly symmetric. \qed
\end{proof}

For such a concept class there is thus hope that a relation like the one from Theorem \ref{T_MQLB} might be useful. However there are interesting concept classes that lead to a weakly symmetric function $\mathcal{F}$. An example is the class of monomials of size exactly $k$:

\begin{theorem}
If $\mathcal{F}_{k}$ is the class of monomials of size exactly k, $\mathcal{F}_{k}$ is weakly symmetric.
\end{theorem}

\begin{proof}
Let $X = \{0,1\}^n$ and $V$ the extended set of $2n$ variables indexed in $[n]$ that contains the variables and their complements: $V = \{y^{(i)}|i \in [2n]\}$ with $y^{(i)} = x^{(i)}$ for $i \in [n]$ and $y^{(i)} = \neg x^{(i)}$ for $i \in [n+1,2n]$. 

Let's fix $x_a, x_b \in X$ and let $I = \{i|x_a^{(i)} = x_b^{(i)}\}$ be the set of variables that have identical values for $x_a$ and $x_b$ and $D = \{i|x_a^{(i)} \not= x_b^{(i)}\}$, the complement of $I$. We will construct a permutation $\sigma^{(a,b)}$ over $V$ based on $I$ and $D$ that will induce a permutation $\pi_\sigma$ over $X$ which will in turn induce a permutation $\Pi_\sigma$ over $2^X$. $\Pi_\sigma$ will have the property that $\mathcal{F}_k(\Pi_\sigma(f)) = \mathcal{F}_k(f), \forall f$  i.e. the permutation that $\sigma^{(a,b)}$ induces on the set of possible functions $f$ leaves $\mathcal{F}$ unchanged, which is what we need to show.

For any $i \in I$, let $\sigma^{(a,b)}(i) = i$ and for any $i \in D$, let $\sigma^{(a,b)}(i) = i + n$. In other words any variable on which $x_a$ and $x_b$ agree will remain unchanged, while any variable for which there is disagreement will be negated. From the construction of $\sigma$ it follows that $\pi_\sigma(x_a) = x_b$ and $\pi_\sigma(x_b) = x_a$, as desired (where, as above, $\pi_\sigma$ is the permutation induced by $\sigma$ on $X$).

We will first consider $f$'s such that $\mathcal{F}_k(f) = 1$, which means $f$ is a monomial of size $k$. In  the expression of $f$, $\sigma$ will either leave a variable unchanged or it will replace it with its negation. But that means that $\Pi_\sigma(f)$ (where as above $\Pi_\sigma$ is the induced permutation over $2^X$) will still be a monomial of size exactly $k$ (albeit a different one), so $\mathcal{F}_k(\Pi_\sigma(f)) = 1$. 

The second case considers functions $f$ such that $\mathcal{F}_k(f) = 0$. Let $s$ be the number of terms in the minimal DNF representation of $f$. Is $s=1$ then $f$ is representable by a monomial and, since $\mathcal{F}_k(f)=0$, $f$ is a $k'$-monomial with $k' < k$ or $k' > k$. But negating any subset of variables from $f$ will not increase or decrease the number of variables in the conjunction (as the variables are uniquely represented in the conjunction and the expression can't be reduced in any way), so $\mathcal{F}_k(\Pi_\sigma(f))=0$ for $s=1$. 

If $s > 1$, let's assume that $\Pi_\sigma(f)$ has an $s'$-term DNF representation for some $s' < s$. But this means that $\Pi_\sigma(\Pi_\sigma(f))$ (we apply the induced permutation $\Pi_\sigma$ a second time) will also have an $s'$-term DNF representation. But since $\Pi_\sigma(\Pi_\sigma(f)) = f$ (as applying $\Pi_\sigma$ two times only doubly negates a subset of the variables), we get a contradiction. So $\Pi_\sigma(f)$ has an $s$-term DNF representation with $s > 1$ and can't be a monomial. Thus $\mathcal{F}_k(\Pi_\sigma(f)) = 0$.
\qed
\end{proof}

It is easy to extend the proof and show that the class of monomials of size at most $k$ also lead to weakly symmetric $\mathcal{F}$ functions.

\section{Discussion}
As mentioned in the introduction, the combination of evaluation and learning is a characteristic of property testing. A natural question is whether we can design an exact (i.e. non-distributional) property testing protocol that is useful. As we saw in Section \ref{S_Positive}, in the exact setting we are considering, whenever we will be able to positively test for membership, the learning problem will be hard and vice-versa. So, as compared to the commonly used property testing protocol (which is defined with respect to some distribution over the instance space), we can't expect two-sided property testers (that certify both membership and non-membership) to be combined with exact learners successfully. But, it is still possible to combine learners and algorithms that certify non-membership with potential applications to agnostic exact learning.

Regarding other future directions, one natural thing to study is a general upper bound on $aC^0$ that only depends on the size of the concept class. Moreover, as mentioned in the text, the lower bounds for $aC^0$ and $aTD$ use different tools to obtain a similar result, and these tools are often encountered in proofs for lower bounds, so perhaps understanding their connections would be beneficial in its own right.

Another interesting research direction is to study bounds for $C^0$ and $aC^0$ for particular concept classes. Several results exist (\cite{heg95} and \cite{hel96}) for $C^0$ but they do not cover all natural concept classes. Another hope is that deriving upper bounds for $C^0$ and $aC^0$ would in turn lead to a deeper understanding of the gap between the worst case upper bound for $aTD$ and the upper bounds for particular concept classes.

On another topic, as described in Section \ref{S_MQ}, interesting connections exist between the membership query learning and deterministic decision tree frameworks. One interesting direction would be to further investigate what other function classes lead to weakly symmetric $\mathcal{F}$ functions, as both positive and negative answers would potentially help in revealing new connections between learning and evaluation. 

\subsection*{Acknowledgments}
The author would like to thank Rocco Servedio, Michael Saks and Chris Mesterharm for their valuable comments and feedback.

{\small
\bibliographystyle{abbrv}
\bibliography{ExactLearning}
}
\end{document}